\documentclass[sigconf]{acmart}

\DeclareMathOperator*{\argmin}{argmin}

\AtBeginDocument{%
  \providecommand\BibTeX{{%
    \normalfont B\kern-0.5em{\scshape i\kern-0.25em b}\kern-0.8em\TeX}}}

\usepackage{amsmath,amsfonts}

\def\w{\mathbf w}
\def\x{\mathbf x}

\def\z{\mathbf z}

\def\a{\mathbf a}
\def\w{\mathbf w}

\def\z{\mathbf z}
\def\x{\mathbf x}

\def\a{\mathbf a}

\def\q{\mathbf q}

\setcopyright{acmlicensed}
\copyrightyear{2023}
\acmYear{2023}
\acmConference[AIEE 2023]{2023 The 4th International Conference on Artificial Intelligence in Electronics Engineering (AIEE)}{January 06--08, 2023}{Haikou, China}
\acmBooktitle{2023 The 4th International Conference on Artificial Intelligence in Electronics Engineering (AIEE) (AIEE 2023), January 06--08, 2023, Haikou, China}
\acmPrice{15.00}
\acmISBN{978-1-4503-9951-7/23/01}
\acmDOI{10.1145/3586185.3586188}

\begin{document}

\title{Why Deep Learning's Performance Data Are Misleading}

\author{Juyang Weng}
\email{weng@msu.edu}
\orcid{0003-1383-3872}
\affiliation{%
  \institution{Brain-Mind Institute and GENISAMA}
  \streetaddress{4460 Alderwood Dr.}
  \city{Okemos}
  \state{MI}
  \country{USA}
  \postcode{48864}
}

%
%
%
%
%
%

\renewcommand{\shortauthors}{Weng}

\begin{abstract}
 This is a theoretical paper, as a companion paper of the keynote talk at the same conference AIEE 2023.   
 In contrast to conscious learning, many projects in AI have employed so-called ``deep learning'' many of which seemed to give impressive performance.
 This paper explains that such performance data are deceptively inflated due to two misconducts:
 ``data deletion'' and ``test on training set''.   This paper clarifies ``data deletion'' and  ``test on training set'' in deep learning and why they are misconducts.   A simple classification method is defined, called
Nearest Neighbor With Threshold (NNWT).   A theorem is established that the NNWT method 
 reaches a zero error on any validation set and any test set using the two misconducts, as long
 as the test set is in the possession of the author and both the amount of storage space and the time of training are finite but unbounded like with many deep learning methods. 
 However, many deep learning methods, like the NNWT method, are all not generalizable since they have never been tested by a true test set.  Why? The so-called ``test set'' was used in the Post-Selection step of the training stage.  The evidence that misconducts actually took place in many deep learning projects is beyond the scope of this paper.  
 \end{abstract}



\maketitle

\section{Introduction}

Since 2012, AI has attracted much attention from public and media.  A large number of projects in AI have published \cite{LeCun15,Bellemare20,Mnih15short,Silver16,Graves16,Silver16,Silver17,Silver18,Moravcik17,Senior20,McKinney20,Schrittwieser20,Senior20,Bellemare20}, including AlphaGo, AlphaGoZero, AlphaZero, AlphaFold, and IBM Debater. 
.  
If the authors of these projects understand the principles in this report, they could benefit much for reducing the time and manpower to reach their target systems as well as improving the generalization powers of their target systems.

So-called ``Deep Learning'' has two steps in the training stage \cite{WengPSUTS21}.
The training stage has two steps:  First, fit multiple systems each starting with random weights 
(and try many hyper-parameters) using a training set.  Second, Post-Select the luckiest fit system based on validation set and test set.   Therefore, so-called ``Deep Learning'' is without any test stage. 

This paper, based on the analysis of Post-Selections \cite{WengPSUTS21}, raises two flaws that 
seem to widely exist in machine learning projects:  (1) data deletion and (2) test on training data.
The latter is applicable when test data are in the possession of the authors.   All authors of published papers that report authors' own tests are in the possession of the test sets.   

An open-competition is different, such as
Deep Blue versus Garry Kasparov Feb. 10, 1996 - May 11, 1997,  AlphaGo versus Lee Sedol March 9, 2016 - March 15, 2026, and AlphaGo versus Ke Jie May 23, 2007 - May 27, 2017, because the test data arrives on the fly.   This paper defines Post-Selections Using Test Set (PSUTS) On The Fly (OTF) conducted by humans behind the scene.  The author does not claim that PSUTS OTF indeed took place during 
any of these three events (although Kasparov did).   The PSUTS OTF mode is for the academic community to be aware of, to be alert about, and to investigate in the future.

In the remainder of the paper, we will discuss four learning conditions in 
Sec.~\ref{SE:Conditions} from which we can see that we cannot just look at superficial ``errors'' without limiting resources. 
Sec.~\ref{SE:4Maps} discusses four types of mappings for a learner, which gives spaces on which we can 
discuss errors. 
Post-Selections are discussed in Sec~\ref{SE:Post}.  Section~\ref{SE:conclusions} provides concluding remarks.

\section{The Four Learning  Conditions}
\label{SE:Conditions}
First, let us consider four learning conditions that any fair comparisons of AI methods should take into account. 

Many AI methods were evaluated without considering how much computational recourses are 
necessary for the development of a reported system.   Thus, comparisons about the performance of the system have been tilted toward competitions about how much resources a group has at its disposal, regardless how many networks have been trained and discarded, and how much time the training takes.  

Here we explicitly define the Four Learning Conditions for development of an AI system:

\begin{definition}[The Four Learning  Conditions]
\label{DF:Conditions}
The Four Learning  Conditions for developing an AI system are:  (1) A body including sensors and effectors, (2) a set of restrictions of learning framework, including whether task-specific or task-nonspecific, batch learning or incremental learning; (3) a training experience and (4) a limited amount of computational resources including the number of hidden neurons. 
\end{definition}

For example, the ImageNet competition \cite{Russakovsky15} did not seem to explicitly restrict conditions 
(2), (3) and (4).  The given images and class labels correspond to restriction in condition (1).
The AIML Contests \cite{WengAAAIFS18} considered all the four in performance evaluation.   

Weng 2021 \cite{WengPSUTS21} discussed the conditions (2) to (4) without condition (1).
It further discussed why any Big Data set violates what is called the sensorimotor recurrence principle, namely, any learning process that uses a static Big Data set is physically flawed. 

\section{Four Different Mappings}
\label{SE:4Maps}

Traditionally, a neural network is meant to establish a mapping $f$ from the space of input $X$ to the space of
class labels $L$, 
\begin{equation}
f:X \mapsto L
\label{EQ:XtoLmapping}
\end{equation}
\cite{Funahashi89,Poggio90a}.   $X$ may contain a few time frames.  

Many temporal problems, such as video analysis problems, speech recognition problems, and computer game-play problems, can include context labels in the input space, so as to learn 
\begin{equation}
f:X \times L \mapsto L.
\label{EQ:XLtoLmapping}
\end{equation}
where $\times$ denotes the Cartesian product of sets. 

A developmental approach deals with space and time in a unified fashion using a neural network such as Developmental Networks (DNs) 
\cite{WengWhy11} 
whose experimental embodiments range from WWN-1 to WWN-9.  The DNs went beyond vision problems to attack general AI problems including vision, audition, and natural language acquisition as emergent Turing machines \cite{WengIJIS15}.
DNs overcame the limitations of the framewise mapping in Eq.~\eqref{EQ:XLtoLmapping} by dealing with lifetime mapping without using any symbolic labels:
 \begin{equation}
f: X(t-1)\times Z(t-1) \mapsto  Z(t), t=1, 2, ... 
\label{EQ:XZmapping}
\end{equation}
where $X(t)$ and $Z(t)$ are the sensory input space and motor input-output space, respectively.

Note that $Z(t-1)$ here is extremely important since it corresponds to the state of a Turing machine.  Namely, all the errors occurred during any time of each life is 
recorded and taken into account in the performance evaluation.   Different from the space mapping in Eq.~\eqref{EQ:XtoLmapping} and very important, the space $Z(t)$ is the directly teachable space for the learning system, inspired by brains \cite{Super76,Thoroughman,Rizzolatti87,Moore03,Thoroughman,Iverson10}.  
This new formulation is meant to model not only brain's spatial processing \cite{WengSpace12}
and temporal processing \cite{WengTime}, but also Autonomous Programming for General Purposes (APFPG)  \cite{WengIJCNN2020,WengIJHR2020}.    Based on the APFGP capability, the AI field seems to have a powerful yet general-purpose framework towards conscious machines \cite{WengCAI-ICDL20}.
 
We need to consider two factors: (A) Space:  Because $X$ and $Z$ are vector spaces of sensory images and muscle neurons, we need internal 
neuronal feature space $Y$ to deal with sub-vectors in $X$ and $Y$ and their spatial hierarchical features.  (B) Time:  Furthermore, considering symbolic 
Markov models, we also need further to model how $Y$-to-$Y$ connections
enable something similar to higher and dynamic order of time in Markov models.  With the two considerations (A) Space and (B) Time, the above lifetime mapping in Eq.~\eqref{EQ:XZmapping} is extended to:
\begin{equation}
f: X(t-1)\times Y(t-1) \times Z(t-1) \mapsto  Y(t) \times Z(t), t=1, 2, ... 
\label{EQ:XYZmapping}
\end{equation} 
in DN-2.  It is worth noting that the $Y$ space is inside a closed ``skull'' so it cannot be directly supervised. 
$Z(t-1)$ here is extremely important since it corresponds to the state of an emergent Turing machine.  

Asim Roy \cite{Roy08} argued that there are some parts of the brain that control other parts.  Here the area $Z$ could be treated as a regulatory ``controller'' that regulates hidden neurons in the $Y$, e.g., as sensorimotor rules.  In neuroscience, there have been many published models that handcraft areas as top-down regulators.  In the DN model below, such areas must be automatically generated and refined ML-optimally inside the closed skull across lifetime.

In terms of performance evaluation, all the errors occurred during any time in Eq.~\eqref{EQ:XYZmapping} of each life is 
recorded and taken into account in the performance evaluation.  This is in sharp contrast with Post-Selection.    

\section{Post-Selections}
\label{SE:Post}
Before we discuss Post-Selections, we need to discuss three types of errors.

\subsection{Fitting, Validation and Test Errors}

Given an available data set $D$, $D$ is divided by a partition into three mutually disjoint sets, a fitting set $F$, a validation set $V$, and a test set $T$ so that
\begin{equation}
D=F\cup V \cup T.
\label{EQ:disjoint}
\end{equation}
Two sets are disjoint if they do not share any elements.   
The validation set is possessed by the trainer, the test set should not be possessed by the trainer since the test should be conducted by an independent agency.  Otherwise, $V$ and $T$ become equivalent. 

Given any hyper-parameter vector $\a_i$ (e.g., including receptive fields of neurons), it is unlikely that a single network initialized by a set of random weight vectors can 
result in an acceptable error rate on the fitting set, called fitting error, that the error-backprop  training intends to minimize locally.
That is how the multiple sets of random weight hyper-parameter vectors come in.  For $k$ hyper-parameter vectors $\a_i$, $i=1, 2, ... k$ and $n$ sets of random initial weight vectors $\w_j$, the error back-prop training results in $kn$ networks 
\[
\{N(\a_i, \w_j) \;|\; i=1, 2, ... , k, j=1, 2, ..., n\} .
\]
Error-backprop locally and numerically minimizes the fitting error $f_{i,j}$ on the fitting set $F$.  

\begin{figure}[h]
	\centering
	\includegraphics[width=\linewidth]{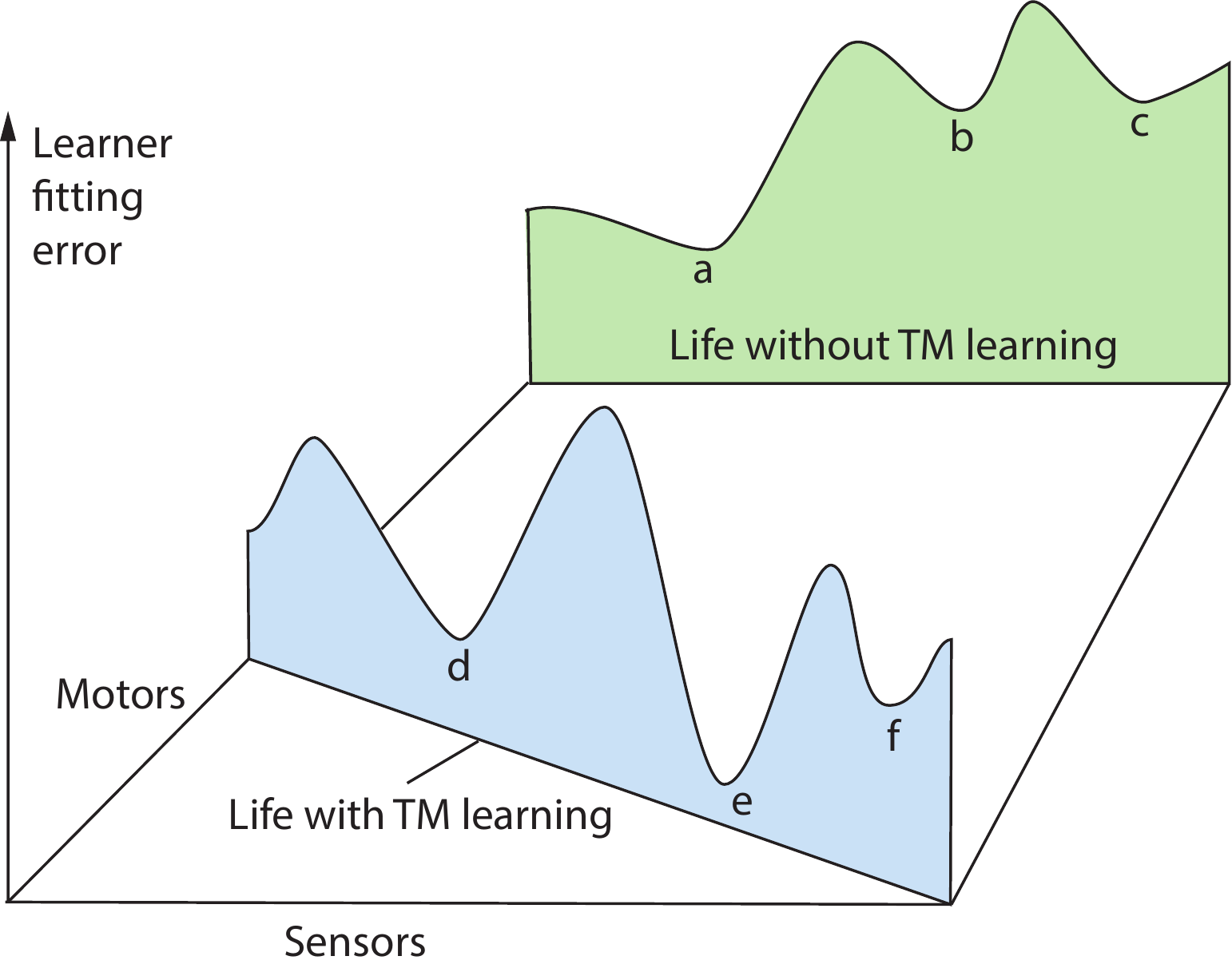}
	\caption{A 2D-terrain illustration for the global minima problems in a high-dimensional terrain (e.g., 200B-dimensional space of hyper parameters and weights).  TM: Turing machine.}
	\label{FG:PSUTS-1D}
\end{figure}

Fig.~\ref{FG:PSUTS-1D} gives a 2D illustration for the limitations of Post-Selection as well as the difference between the sensor-only mapping in Eq.~\eqref{EQ:XtoLmapping} and the sensorimotor mapping in Eq.~\eqref{EQ:XZmapping}.  Suppose $X$ represents the space of sensory input and $Z$ represents the
space of motor input, but in general should be any initial pair $(\a_i, \w_j)$.  Each location on the 2D plane of Eq.~\eqref{EQ:XtoLmapping} corresponds to 
the initial random weights of a neural network (e.g., 200B-dimensional).   Of course, each network has 
many neurons not just two input values but Fig.~\ref{FG:PSUTS-1D} can only schematically 
represent the initial weights of two values as a 2D-terrian illustration.    
The height of a curve at is the system fitting error $e_{i,j}$ of a particular trained network $N(\a_i, \w_j)$.
We use the biological term ``life'' to indicate the entire process of a system's learning.  

The error-backprop learning method is a greedy method.  Starting from any initial pair $(\a_i, \w_j)$ it steps 
along the direction that descends the quickest without knowing where the global minimum is in the
200B-dimensional space.   In Fig.~\ref{FG:PSUTS-1D}, we can see that $(\a_i, \w_j)$ leads to a local minimum $a$, $b$, or $c$ if we use the sensory mapping in Eq.~\eqref{EQ:XtoLmapping}.   The point $a$ is the lowest point, but there is no guarantee to reach it, depending on where the network starts from in the 200B-dimensional space.  The more networks have been trained, the more likely the luckiest network 
finds the global minimum.   In Fig.~\ref{FG:PSUTS-1D}, the luckiest network starts from the valley 
where $a$ is located, but this is much harder for the 200B-dimensional space.  Typically, the more networks a project has trained, the more likely for the Post-Selection stage to find a network with a smaller $e_{i^*,j^*}$.

Graves et al. \cite{Graves16} seems to have mentioned that the number of trained systems is at least $n=20$.  Saggio et al. \cite{Saggio21} reported that $n$ is at least $10,000$.  Krizhevsky \& Hinton \cite{Krizhevsky17} did not give $n$ but seems to have mentioned 60 million parameters which probably means each $\w_i$ and each $\a_j$ combined to be of
60 million dimensional.   Consider a small example:  The number of tried value $l$ for each hyper-parameter: $l=3$, and the number of hyper-parameters $d=10$, the total number of hyper-parameter vectors is $k=l^d=3^{10}=59049$.  Letting $n=20$, $kn=l^dn \approx1$M networks must be trained, a huge number that requires a lot of computational resources to do number crunching and a lot of manpower to manually tune the range of hyper-parameters!

\begin{definition}[Distribution of fitting, validation and test errors]
The distributions of all $kn$ trained networks' fitting errors $\{f_{ij}\}$, validation errors $\{e_{ij}\}$, and test errors $\{e'_{ij}\}$, $i=1, 2, ... k$, $j=1, 2, ... n$ are random distributions depending on a specific data set $D$ and
its partition $D=F\cup V\cup T$.  The difference between a validation error and a test error is that the former is computed from the same author using 
an author-possessed validation set $V$ but the latter is computed by an independent agency using an 
author-unknown test set $T$.
\end{definition}

We define a simple system that is easy to understand for our discussion to follow. 
Consider a highly specific task of recognizing patterns inside the annotated windows in Fig.~\ref{FG:ImageNet-Annotation}.  This is a simplified case of the three tasks---recognition (yes or no, learned patterns at varied locations and scales), detection (presence of, or not, learned patterns) and segmentation (of recognized patterns from input).  These three tasks of natural cluttered scenes were dealt with by  
the first deep learning networks for 3D---Cresceptron~\cite{WengCresIJCV97}.  Later data sets like ImageNet \cite{Russakovsky15} contain many more image samples.
\begin{figure}[tb]
	\centering
	\includegraphics[width=\linewidth]{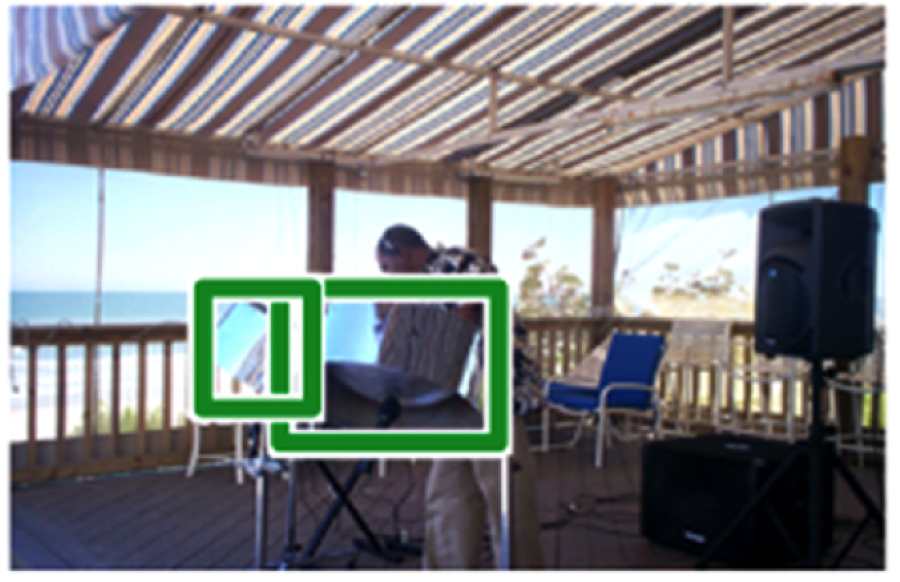}
	\caption{
	Two annotated windows for an object class labeled as ``steel drum'' for single object localization.  Figure courtesy of \cite{Russakovsky15}.
	}
	\label{FG:ImageNet-Annotation}
\end{figure}

\subsection{Post-Selection Illusion}
To understand Post-Selection in Deep Learning (using CNN, LSTM, etc.), including supervised learning, reinforcement learning and adversary learning, let use consider a simple classifier that shares the same principle of Post-Selection but does not have the distraction of details in more commonly used networks as well as learning modes. 

\begin{definition}[Nearest neighbor with threshold, NNWT]
\label{DF:NN}
Define a network that stores the entire fitting set $F$ where each image in $F$ may contain multiple annotated windows for matching (see Fig.~\ref{FG:ImageNet-Annotation}).  For each input image $q$, a scan subwindow $x$ searches across the input image $q$ for a range of locations and scales, normalizes the scale, and compares with each annotated window in the fitting set $F$.   Suppose an input window $x$ matches the nearest sample (annotated) window $s$ in $F$.  If the distance between $x$ and $s$ is not larger than a distance threshold $d$ (a hyper-parameter), then 
the network outputs the associated label of the nearest sample $s$. 
Otherwise, the system outputs a label randomly and uniformly sampled from the label set $L$.
\end{definition}  

Let us give some formality of the Post-Selection with NNWT.
Suppose that there are three data sets, fitting set $F$, verification set $V$, and 
test set $T$, a desired verification error $e_v \ge  0$, and a desired test error $e_t \ge 0$, run the following program $P(s)$ that starts from random seed $s$.  
\begin{enumerate}
\item Store all data from the fitting set $F$. 
\item For each receptive field (e.g., rectangular) $\q$ inside of a cluttered image or a cluttered video $\x \in V$, or $\x \in T$, compute the distance $d_i=d(\q, \x_i)$ between receptive field $\q$ and each annotated receptive field $\x_i$ in $F$ for all values of $i$ (each sample in $F$ have multiple annotated receptive field $\x_i$'s).  
\item Produce the class label $l_j$ of the nearest neighbor $\x_j$ that produces the smallest distance: $j =\argmin_{i} \{ d_i \}$. 
\item If $d_j \le d(s)$ where $d(s)$ is a threshold depending on random seed $s$, output the label  $l_j$ of the nearest neighbor.  Otherwise, output a label that is randomly and uniformly sampled from the label set $L$.
\end{enumerate}

Post-Selection stage:  If the measured verification error or the test error is larger than required, do the above procedure $P(s)$ for a new seed $s$.  Repeat until the measured verification error and test error are both not larger than the required $e_v $ and $e_t$, respectively.  

In order words, the NNWT classifier with Post-Selections uses a lot of space and time resources for over-fitting.  It has a perfect fitting error (zero) but it randomly guesses an output label if the distance is larger than the threshold $d$.

To understand why Post-Selections give misleading results, let us derive the following important theorem.
\begin{theorem}[Post-Selection Illusion]
\label{TM:illusion}
Given any validation error rate $e_v \ge 0$ and test error rate $e_t \ge 0$, using Post-Selections the NNWT classifier using Post-Selections satisfies the $e_v$ and $e_t$, if the author is in the possession of the test set and both the storage space and the time spent on the Post-Selections are finite but unbounded.
\end{theorem}

\begin{proof}
Because the number of seeds to be tried 
during the Post-Selection is unlimited, we can prove that there is a finite time at which a lucky seed $d(s)$ will 
produce the good enough verification error and test error.  
Although the waiting time is long, the time is finite because $V$ and $T$ are finite.  Let us formally prove this.  Suppose $l$ is the number of labels in the output set $L$ and for the set of queries in $V$ and $T$, there are $k$ outputs that must be guessed.  The probability for a single guess to be correct is $1/l$ due to the uniform sampling in $L$.  The probability for $k$ guesses to be all correct 
is  $(1/l)^k=1/l^k$ because guesses are independent.  For an independently and randomly initialized network to guess less than $k$ cases correct is 
$1- 1/l^k$, with $0< 1- 1/l^k< 1$.  The probability for as many as $n$ independently and randomly initialized networks, all of which do not satisfy the $e_v $ and $e_t$, is
\[
p(n) = (1- 1/l^k)^n \longrightarrow 0, \mbox{as $n \longrightarrow \infty$}
\]
because $0< 1- 1/l^k< 1$.  Therefore, within finite time span, a process of trying incrementally more networks will get a lucky network that satisfies both the computable $e_v $ and $e_t$. 
\end{proof}

As we can see from the proof, the smaller the threshold, the more guesses must be made and, thus, typically the longer time one needs to spend during Post-Selections.

Let us imagine that the threshold  $d(s)$ gradually increases from zero toward infinity.  The nearest neighbor classifier changes from a total-lottery scheme (when $d(s) = 0$) to a traditional nearest neighbor classifier (without label guesses) when all query inputs are beyond the threshold. 

Theorem~\ref{TM:illusion} has established that Post-Selections can even produce a perfect classifier that gives a zero validation error and a zero test error! Yes, while the test sets are in the possession of authors, the authors could show any superficially impressive validation error rates and test error rates (including zeros!) because they used Post-Selections without a limit on resources (to store all data sets and to search for a lucky network).

The above theorem means any comparisons without an explicit limit on storage and time spent on Post-Selections are meaningless, like ImageNet and many other competitions.
It is of course time consuming for a program to search for a network whose guessed labels are 
good enough. But such a lucky network will eventually come within a finite time!     


\begin{corollary}[Misleading AI papers]
If a paper trains more than one system and the author is in the possession of test set, the performance data from the paper are misleading if the paper does not report the number of systems trained in Post-Selections, the amount of computational resources (e.g., the amount of storage, the number of computations), the amount of waiting time, along with the fitting errors, the validation errors, and the test errors of {\em all} trained systems.  The generalization power of the reported system is still unknown. 
\end{corollary}
\begin{proof}
From Theorem~\ref{TM:illusion}, if Post-Selections are allowed, a NNWT can satisfy any nonzero validation error and any nonzero testing error using Post-Selections, since the training set, validation set and the test set are all in the possession of the authors.  The generalization power of all trained systems is unknown, so is the reported luckiest system. 
\end{proof}

This corollary is a scientific basis for the author to raise a violation of protocols and a lack of transparency about the Post-Selection stage in almost all machine learning papers appeared in {\em Nature}, {\em Science}, Communication of ACM \cite{Bengio21} and other publication venues since around 2015. Since all the fitting sets, validation sets and test sets are in the possession of the authors, many papers have claimed misleading results using a flawed protocol.  

Following the terminology of \cite{DudaHartStork}, we use the term ``Test on Training Set'' below. 
\begin{theorem}[Test on Training Set]
Post Selections that used test set amount to tests on training set even though trained networks did not ``see'' the training set during their error backprop training.
\end{theorem}
\begin{proof}
Originally, the term ``Test on Training Set'' means the training stage involves the test set.  Although all trained networks did not ``see'' the test set, the Post-Selection is part of the training stage that produces the reported network and its error.  Since the Post-Selection uses the test set to pick up the luckiest network among 
all trained networks using their errors on the test set, all trained networks see the test set during the
post-selection stage.   
\end{proof}

A typical neural network architecture has a set of hyper-parameters represented by a vector $\a$, where each component corresponds a scalar parameter, such as convolution kernel sizes and stride values at each level of a deep hierarchy, the neuronal learning rate, and the neuronal learning momentum value, etc.  Let $k$ be a finite number of grid points along which such hyper-parameter vectors need to be tried,  $A = \{ \a_i \;|\; i=1, 2, ..., k\}$.  Let's give more detail to the above example.  Suppose there are 10 scalar parameters in each vector $\a_i$.  For each scalar parameter $x$ of the 10 hyper parameters, 
we need to validate the sensitivity of the system error to $x$.   With uncertainty of $x$, we estimate its initial value as the mean $\bar{x}$, positively perturbed estimate $\bar{x}+\sigma_x$ ($\sigma$ is the estimated standard deviation of $x$), and negatively perturbed estimate $\bar{x}-\sigma_x$.  If each scalar hyper parameter has three values to try in this way, there are a total of $k=3^{10}=59049$ hyper-parameter vectors to try, a very large number.   For example, in NNWT, the threshold $\bar{d}$ can be estimated by the average of nearest distance between a sample in $V$ and the nearest neighbor in $F$ and the $\sigma_d$ be estimated by the standard deviation of these nearest distances. 
 
\subsection{Test Data Available in Batch}
Ideally, test sets should not be in the possession of authors, e.g., during a blind test.  However, this is often not true since the authors may use publically available data sets that include test sets. 

\begin{definition}[Post selection in batch]
A human programmer trains multiple systems 
using the fitting set $F$.   After these systems have been trained, the experimenter post-selects a system by searching, manually or assisted by computers, among trained systems based on the batch validation set $V$ (or the batch test set $T$).  This is called Post-Selection in batch---selection of one network from multiple trained and verified (or tested) networks.
\end{definition} 

\begin{definition}[Data deletion in post-selection]
Data deletion is a misconduct during which an author deletes the performance data of some 
bad-performing networks during post-selection, not reported in the corresponding project report. 
\end{definition}
Data deletions are violations of the well-known statistical protocol of cross-validation.   In a minimally acceptable form, the distribution of the performances of all trained networks should be reported by the worst 
error, the average error and the best error among all the trained networks.   Only reporting the best error amounts to misconduct known as data deletion, deleting the performance data of all networks that 
the author does not like.

Obviously, a post-selection wastes (deletes) all trained systems except the selected one.   As we can 
predict 
\cite{WengPSUTS21}, a system from the post-selection tends to have a weaker generalization power
than the reported luckiest error indicates, as illustrated in Fig.~\ref{FG:PSUTS-1D}. 

A Post-Selection in batch can use the validation set $V$ or the test set $T$.   However, if both sets are in the 
possession of the human programmer, the difference between $V$ and $T$ almost totally vanishes under Post-Selections. 

\begin{definition}[PSUTS and PSUVS]
A Machine PSUTS is defined as follows:
 If the test set $T$ is available to the author, suppose the test error of $N(\a_i, \w_j)$ is $e_{i, j}$ on the test set $T$, find the luckiest network $N(\a_{i^*}, \w_{j^*})$ so that it reaches the error of the luckiest hyper-parameters and the luckiest initial weights from Post-Selection Using Test Set (PSUTS): 
\begin{equation}
e_{i^*,j^*} = \min_{1\le i \le k} \min_{1\le j \le n} e_{i,j}
\label{EQ:V}
\end{equation}
and report only the performance $e_{i^*,j^*}$ but not the performances of other remaining $kn-1$ trained neural networks.  PSUVS, V for validation, is similarly defined. 
\end{definition} 

Set $T$ is like set $V$ since it is available. Similarly, a human PSUTS is a procedure wherein a human selects a system from multiple trained systems 
according to $\{ e_{i,j} \}$ using also human visual inspection of internal representations of the system and their test errors.   

\subsection{Cross-Validation}

The above PSUTS is an absence of cross-validation \cite{JainDubes}.   Originally, the cross-validation is meant to mitigate an unfair luck 
in a partition of the dataset $D$ into a fitting set $F$ and a test set $T$ (empty validation set).  For example, an unfair luck 
is such that every point in the test set $T$ is well surrounded by points in the fitting set $F$.  But such a luck is hardly true in reality. 

To reduce the bias of such a luck, an $n$-fold cross-validation protocol \cite{DudaHartStork} is suggested.  

\subsection{Types of Lucks in a Neural Network}

 In a neural network, there are at least three types of lucks:

{\bf Type-1 order lucks}: The luck in a partition $P_i$ into a fitting set $F_i$ and a test set $T_i$ from a data set $D$ resulting in test error $e_i$, $i=1,2, ..., n$.   
Different partitions correspond to different luck outcomes.  This kind of outcome variation results in a variation of performance from different outcomes.   Conventionally, this type of lucks is filtered out by cross-validation (e.g., $n$-fold cross-validation) as well as reporting the deviation of $\{e_i\}$ during the cross-validation.  However, such cross-validation and deviation have hardly published for neural networks and reported.   The smaller the average $\bar{e}$ of $\{e_i\}$, the more accurate the trained network is; the smaller the standard deviation of $\{e_i\}$, the more trustable the average error $\bar{e}$ is. 

{\bf Type-2 weights lucks}: As discussed in \cite{WengPSUTS21}, weights specify the role assignment for 
all the neurons in the neural network.  A random seed value determines the initialization of  a pseudo-random number generator, which gives initial weights $\w_i$ for a neural network $N(\w_i)$, resulting in a test error $e_i$, $i=1,2, ... , n$, after training of these $n$ networks and testing on $T$.   
It is unknown that such a 
luck will be carried over to a new test set $T'$ that is outside the data set $D$ but was drawn from the same distribution of $S$.  Because a neural network might not capture the internal rules of the fitting set $F$, this paper argues that a statistical validation of the reported error should be performed by reporting the distribution of $\{e_i | i=1, 2, ... n\}$, where $e_i$ is from a different initial weight vector $\w_i$.   For example, Krizhevsky et al. \cite{Krizhevsky17} reported 60 million parameters, mostly in $\w_i$ but only the luckiest  $e_i$ was reported.  The smaller the average $\bar{e}$ of $\{e_i\}$, the more accurate the trained network is; the smaller the standard deviation 
$\sigma$ of $\{e_i\}$, the less sensitive the trained neural network is to the initial weights and thus the accuracy is more trustable for real applications.  For i.i.d. (identically independently distributed) errors, 
we can expect that doubling the number $n$ will reduce the expected variance of $\bar{e}$ by a factor $1/\sqrt{2}$., since the expected variance of $n$ random numbers is about $\sigma^2/n$.
 
{\bf Type-3 hyper-parameter lucks}: Each hyper-parameter vector $\a_j$ of the neural network gives an error $e_j$, $j=1,2, ..., k$.   Because such a luck of $\a_j$ might not capture the internal rules of the fitting set $F_j$, this paper argues that a statistical validation of the reported error estimate should be performed and the distribution of $\{e_j\}$ be reported.   
 In our above example, the number of distinct hyper-parameter vectors to be tried is $k=3^{10}=59049$.   The smaller the average $\bar{e}$ of $\{e_j\}$, the more accurate the trained network is; the smaller the sample variance of $\{e_j\}$, the more trustable $\bar{e}$ is, namely, 
the average error $\bar{e}$ is less sensitive to the initial hyper-parameters of the network.    For example, 
the threshold $d$ of the nearest neighbor classifier in Definition~\ref{DF:NN} might result in a large deviation. A good way is to reduce the manual selection nature of such hyper-parameters.  For example, all hyper-parameters are adaptively adjusted from the initial hyper-parameters that are further automatically computed 
from system resources, e.g., the resolution of a camera, the total number of available neurons, and the firing age of each neuron \cite{WengLCA09}. 

For notation clarity in the discussion that follows, index $j$ is used in Type 3 to distinguish index $i$ in type 2, but the above three types of lucks are all different. 

Let us discuss the case of a developmental network, such as Cresceptron \cite{WengCresIJCV97} and DN \cite{WengIJIS15}.  Type-1 cross-validation is not needed because of reporting of a lifetime error.
In other words, errors of all new tests in each life are taken into account throughout the lifetime.
Type-2 validation is not needed because all different random weights $\w_i$ leads to the function-equivalent  neural network under certain conditions.  For example, in top-$k$ competition, with $k=1$ different 
$\w_i$ give the exactly the same neural network and with $k>1$ different 
$\w_i$ give almost the same neural network.   The distribution of lifetime errors $\{e_i\}$ is expected to have a negligible deviation across different initial weight vectors $\w_i$, given the same Four Learning Conditions.
Type-3 validation might be useful but is expected to be negligible since the most obvious parameters such as learning rate and momentum of learning rate is automatically and optimally determined by each neuron, 
not handcrafted, as in LCA \cite{WengLCA09}.   The synaptic maintenance automatically adjusts all receptive fields \cite{Wang11,GuoIJCNN14} so that the neural network performance is not sensitive to the initial hyper-parameters. 
 
In contract, a batch-trained neural network typically uses a Post-Selection to pick the luckiest network without cross-validation for either of the above three types of lucks, e.g., in ImageNet Contest \cite{Russakovsky15}.  Namely, errors occurred during batch training of the network before the network is finalized and how long the training takes are not reported.   Many researchers have claimed error-backprop ``works'' without providing much-needed three types of validations.   This seems not true since \cite{ZhengCVVT16} shows a huge difference between the luckiest CNN with error-backprop and the optimal DN.

We also need to be aware of another protocol flaw:  Random seeds and hyper parameters are all coupled. 
Under such a coupling, Type 2 validation seems unnecessary with $n=1$ but the search of the luckiest weights is embedded into the search for the luckiest hyper-parameter vector.
where each hyper parameter vector uses a different seed.  

Since a PSUTS procedure picks the best system based on the errors on the test set (like a validation set), the resulting system does not do well on new test sets because doing well on a validation set does not guarantee doing well on an open test set.  See Fig.~\ref{FG:PSUTS-1D}.  Typically, due to a very large number of samples, availability of test sets and unavailability of open test sets in a properly managed contest, Post-Selections cause the reported error to be smaller than an open test error rate.   (However, in Table 2 of \cite{Krizhevsky17}, the test error rate is smaller than the validation error for 7CNNs, causing a reasonable suspicion that PSUTS could be used in addition to PSUVS.)  

 
\subsection{The Luckiest Network from a Validation Set: Data Deletion}

Many people may ask:  Are there any technical flaws in at least PSUVS, since it does not use the test sets?  Any post-selection is technically flawed and results in misleading results, including both PSUVS and PSUTS, because of the data deletion misconduct.  

In general, Type-1 cross-validation is to filter out lucks in data partition that a typical user does not have during a deployment of the method.   Namely, it is a severe technical and protocol flaw in reporting only the luckiest network, regardless the post-selection uses validation sets or test sets.   At least the average
error over Post-Selections must be reported. 

This conclusion has a great impact on evolutional methods that often report only the luckiest network, instead of those of all networks in a population.
Namely, the performances of all individual networks in an evolutionary generation should be reported. 
Furthermore, a reasonably disjoint test set must be used to evaluate the generalization of the luckiest network.

If the test set $T$ is in the possession of the author, which seems to be true for almost all so-called ``deep learning'' publications other than open competitions, we define machine PSUTS as a Post-Selection process where test set is used, in addition to the validation set. 

Some researchers have claimed that the test data was ``unseen'' by trained systems when they were tested, but the network selector has seen their performances when they see the test set.  In other words, 
the Post-Selection stage belongs to the training phase---training the Post-Selector. 

Weng \cite{WengPSUTS21} discussed also human PSUTS where it is a human that does Post-Selection.  The subsection below discusses human PSUTS on the fly, not in batch.  

Weng \cite{WengPSUTS21} also discussed why error-backprop needs machine PSUTS and human PSUTS.

\subsection{Open Test Data Arriving on the Fly}

\begin{definition}[PSUTS On The Fly]
During an open competition, the participating $m>1$ luckiest networks meet new open test data that arrive on the fly.  One or more humans conduct an additional round of Post-Selections on the fly during the competition using human PSUTS from the $m$ machine outputs for some actions.  Here, all human selected actions across the entire game do not even have to be consistently from a single network. 
\end{definition} 

For example, $\x(t)$ is the current board configuration in a Go game, treated as a game state 
at player time $t$.  Human experts (behind the scene or on the scene), based on the output actions from the $m$ luckiest networks, manually select a network's output as the next action $\z(t+1)$.   Of course, this is not a fair competition, because the so-called ``human player side'' has only a single human who is not allowed 
to use any computer but the so-called ``machine player side'' has one or more humans who have assistance from computers.

Weng \cite{WengPSUTS-ICDL21} discussed how the DN algorithm trains only a single network that is always optimal in the sense of maximum likelihood, conditioned on the Four Learning Conditions.  In this sense, the DN is free from local minima.

\section{Conclusions}
\label{SE:conclusions}

We used intuitive terms but formal ways to discuss Post-Selections.   
So-called ``Deep Learning'' involves two misconducts, ``data deletion'' and ``test on training data''.  Because it is without a test stage, ``Deep Learning'' is not generalizable.  Performance data from so-called ``deep learning'' are misleading without explicit exclusion of such flaws in their experimental protocol.  Such flawed protocols are tempting to those published papers where the test sets are in the possession of the authors and also to open-competitions where human experts are not explicitly disallowed to interact with the ``machine player'' on the fly.  Evidence of such misconduct is referred to Weng et al. v. NSF et al. MWDC  1:22-cv-998.

Public and media have gained an impression that deep learning has approached or even ``sometimes exceeded'' human level performance on certain tasks.  For example, the image classification errors from a static image set were compared with those of humans \cite[A2, p242]{Russakovsky15}) and the work is laudable.  However, this paper raises Post-Selections, which seem to question such claims since a real human does not have
the luxury of Post-Selections.  The author hopes that the exposure of Post-Selections is beneficial to AI credibility and the future healthy development of AI, especially with the concepts of developmental errors and the framework of ML-optimal lifetime learning for invariant concepts under the Four Learning Conditions.   Some researchers have raised that it seems that those who wan a competition were those who have more computational resources and manpower at their disposal.   The new developmental error metrics under the Four Learning Conditions hopefully encourages 
future AI competitions to compare methods under the same Four Learning Conditions.  Considering DN as a much-simplified model for a biological machine, it seems not baseless to guess that 
each biological brain is probably ML-optimal (of course in a much richer sense) across lifetime,
e.g., due to the pressure to compete at every age.   The Four Learning Conditions explicitly include other factors that greatly affect machine learning performances such as learning framework (e.g., task-nonspecificity, incremental learning, the robot bodies), learning 
experiences and computational resources.   The analysis that any ``big data'' sets are nonsalable does not
mean that we should not create, use and share data sets.   Instead, we need to pay attention to the fundamental limitations of any static data sets, regardless how large their apparent sizes are.

\balance
\bibliographystyle{ACM-Reference-Format}


%
%
%
%

\end{document}